\documentclass[conference]{IEEEtran}
\IEEEoverridecommandlockouts
\usepackage{cite}
\usepackage{amsmath,amssymb,amsfonts}
\usepackage{algorithmic}
\usepackage{graphicx}
\usepackage{textcomp}
\usepackage{xcolor}
\def\BibTeX{{\rm B\kern-.05em{\sc i\kern-.025em b}\kern-.08em
    T\kern-.1667em\lower.7ex\hbox{E}\kern-.125emX}}

\newcommand{\reals}{\mathbb{R}}

\newcommand{\E}{\mathbb{E}}

\usepackage{arydshln,color,booktabs,amssymb,amsmath,amsthm}
\usepackage[bookmarks=false]{hyperref}

\makeatletter
\def\adl@drawiv#1#2#3{%
        \hskip.5\tabcolsep
        \xleaders#3{#2.5\@tempdimb #1{0.8}#2.5\@tempdimb}%
                #2\z@ plus1fil minus1fil\relax
        \hskip.5\tabcolsep}
\newcommand{\cdashlinelr}[1]{%
  \noalign{\vskip\aboverulesep
           \global\let\@dashdrawstore\adl@draw
           \global\let\adl@draw\adl@drawiv}
  \cdashline{#1}
  \noalign{\global\let\adl@draw\@dashdrawstore
           \vskip\belowrulesep}}
\makeatother

\usepackage{multirow}

\renewcommand{\eqref}[1]{Eq.~(\ref{eq:#1})}

\newcommand{\secref}[1]{Section \ref{sec:#1}}

\newcommand{\lemref}[1]{Lemma \ref{lem:#1}}

\newcommand{\appref}[1]{Appendix \ref{ap:#1}}

\usepackage{wrapfig}

\newcommand{\citet}[1]{\cite{#1}}
\newcommand{\citep}[1]{\cite{#1}}

\newtheorem{definition}{Definition}
\newtheorem{lemma}{Lemma}

\begin{document}

\title{Fast Single-Class Classification and \\the Principle of Logit Separation}

\author{

\IEEEauthorblockN{Gil Keren}
\IEEEauthorblockA{
\textit{ZD.B Chair of Embedded Intelligence} \\
\textit{for Health Care and Wellbeing} \\
\textit{University of Augsburg, Germany}\\
gil.keren@informatik.uni-augsburg.de}

\and

\IEEEauthorblockN{Sivan Sabato}
\IEEEauthorblockA{
\textit{Department of Computer Science } \\
\textit{Ben-Gurion University of the Negev}\\
Beer-Sheva, Israel \\}

\and

\IEEEauthorblockN{Bj\"orn Schuller}
\IEEEauthorblockA{
\textit{ZD.B Chair of Embedded Intelligence} \\
\textit{for Health Care and Wellbeing} \\
\textit{University of Augsburg, Germany}\\
\textit{GLAM -- Group on Language, } \\
\textit{Audio \& Music} \\
\textit{Imperial College London, UK}}

}

\newcommand{\lsplong}{Principle of Logit Separation} 
\maketitle

\newcommand{\tabresone}{
\begin{table*}[htbp]
\caption{Results on Single Logit classification (SLC), using the different loss functions. Lower values are better. In almost all cases, loss functions that are aligned with the \lsplong\ (under the dashed line) yield a mean relative improvement of at least 20\% in SLC accuracy measures, and sometimes considerably more. 
}
\begin{center}

  \begin{tabular}{clcccc}
    \toprule
	Dataset & Method & 1-AUPRC & 1-Precision@0.9 & 1-Precision@0.99  \\
    \midrule[1pt]
	\multirow{8}{*}{\parbox{3cm}{\centering MNIST}}  
 & CE & 0.008 & 0.005 &0.203 \\
 & max-margin & 0.012 & 0.018 &0.262 \\
  \cdashlinelr{2-6} 
 & self-norm 		& 0.002 				& 0.001 		& \textbf{0.021} 	 \\
 & NCE 				& 0.002 				& 0.002 		& \textbf{0.021}  \\
 & binary CE 		& 0.002 				& \textbf{0.000} 		&0.037 	\\
 & batch CE 			& \textbf{0.001} 	& 0.001 		&0.022  \\
 & batch max-margin 	& 0.002 				& 0.001 		&0.034  \\
 \cmidrule{2-6} 
 & Mean relative improvement & 82.0\%	& 91.3\% & 88.4\% \\
   \midrule[1pt]
   \multirow{8}{*}{\parbox{3cm}{\centering SVHN}}
 & CE & 0.023 & 0.028 &0.545 \\
 & max-margin & 0.021 & 0.025 &0.532 \\
   \cdashlinelr{2-6}
 & self-norm & \textbf{0.015} & 0.014 &0.298 \\
 & NCE & 0.021 & 0.017 &0.320  \\
 & binary CE & \textbf{0.015} & 0.016 &0.312 \\
 & batch CE & \textbf{0.015} & \textbf{0.013} & \textbf{0.280} \\
 & batch max-margin & 0.018 & 0.020 &0.384 \\
  \cmidrule{2-6} 
 & Mean relative improvement & 23.4\%	&	39.6\%	&	40.8\% \\
\midrule[1pt]
   \multirow{8}{*}{CIFAR-10} 
 & CE & 0.109 & 0.326 &0.703  \\
 & max-margin & 0.094 & 0.285 &0.705 \\
   \cdashlinelr{2-6} 
 & self-norm & 0.073 & 0.204 &0.599  \\
 & NCE & 0.081 & 0.214 &\textbf{0.594} \\
 & binary CE & \textbf{0.070} & 0.210 &0.607 \\
 & batch CE & 0.072 & \textbf{0.202} &0.602  \\
 & batch max-margin & 0.075 & 0.226 &0.636 \\
  \cmidrule{2-6} 
 & Mean relative improvement & 26.9\%	&	30.9\%	&	13.6\% \\
\midrule[1pt]
   \multirow{8}{*}{CIFAR-100} 
 & CE & 0.484 & 0.866 &0.974  \\
 & max-margin & 0.490 & 0.893 &0.977 \\
   \cdashlinelr{2-6} 
 & self-norm & 0.378 & 0.807 &0.970  \\
 & NCE & 0.383 & \textbf{0.795} &0.964  \\
 & binary CE & 0.426 & 0.870 &0.978  \\
 & batch CE & \textbf{0.371} & \textbf{0.795} &\textbf{0.961}  \\
 & batch max-margin & 0.468 & 0.903 &0.983  \\
  \cmidrule{2-6} 
 & Mean relative improvement & 16.8\%	& 5.2\%	&	0.5\% \\
\midrule[1pt]
 \multirow{3}{*}{\parbox{3cm}{\centering Imagenet\\ (1000 classes)\\($6\cdot 10^6$ iterations)}}
 & CE 	& 0.366 & 0.739 & 0.932  \\
   \cdashlinelr{2-6} 
 & batch CE 		& \textbf{0.245} & \textbf{0.563}  & \textbf{0.865}  \\
  \cmidrule{2-6} 
 & Relative improvement & 33.1\%	& 23.8\%	&  7.2\% \\
    \bottomrule\\
  \end{tabular}
\label{tab:resultsone}
\end{center}
\end{table*}
}

\newcommand{\tabrestwo}{
\begin{table*}[htbp]
\caption{Comparing binary classification with a single logit (SLC) vs.~all logits. Lower values are better. PoLS-aligned SLC methods are above the dashed line. Results are comparable, thus SLC does not cause any degradation in binary classification accuracy, compared to the case where all logits are computed. 
 }
\begin{center}
  \begin{tabular}{clcccc}
    \toprule
	Dataset & Method & 1-AUPRC & 1-Precision@0.9 & 1-Precision@0.99 \\
    \midrule[1pt]
	\multirow{2}{*}{\parbox{3cm}{\centering MNIST}}  
 & Mean PoLS methods & 0.002 & 0.001 & 0.027 \\
  & batch CE 			& 0.001 	& 0.001 		&0.022  \\
  \cdashlinelr{2-5}
 & CE with all logits & 0.001 & 0.000 &0.020 \\
   \midrule[1pt]
   \multirow{2}{*}{\parbox{3cm}{\centering SVHN}}
 & Mean PoLS methods & 0.017 & 0.016 & 0.319 \\
  & batch CE & 0.015 & 0.013 & 0.280 \\
    \cdashlinelr{2-5}
 & CE with all logits & 0.015 & 0.016 &0.313 \\
\midrule[1pt]
   \multirow{2}{*}{CIFAR-10} 
 & Mean PoLS methods & 0.074 & 0.211 & 0.608 \\
  & batch CE & 0.072 & 0.202 &0.602  \\
    \cdashlinelr{2-5}
 & CE with all logits & 0.074 & 0.214 &0.648\\
\midrule[1pt]
   \multirow{2}{*}{CIFAR-100} 
 & Mean PoLS methods & 0.405 & 0.834 & 0.971 \\
  & batch CE & 0.371 & 0.795 & 0.961  \\
    \cdashlinelr{2-5}
 & CE with all logits & 0.380 & 0.801 &0.973 \\
\midrule[1pt]
 \multirow{2}{*}{\parbox{3cm}{\centering Imagenet}}
 & batch CE		& 0.245 & 0.563 & 0.865 \\
   \cdashlinelr{2-5}
 & CE with all logits 	& 0.223 & 0.566 & 0.872 \\
    \bottomrule\\
  \end{tabular}
\label{tab:resultstwo}
\end{center}
\end{table*}
}

\newcommand{\tabresthree}{
\begin{table}[htbp]
\caption{Speedup experiment results. When the number of examples is large, SLC results in a considerable speedup.
}
\begin{center}
  \begin{tabular}{clcccc}
    \toprule
	Architecture & Classes & Inference Time [s] & SLC Speedup \\
    \midrule[1pt]
\multirow{5}{*}{\parbox{2cm}{\centering Alexnet}}
 & $1$ (SLC) & $3.6 \cdot 10^{-3}$ & --- \\
 & $2^{10}$ & $3.7 \cdot 10^{-3}$ & x1.04 \\
 & $2^{14}$ & $4.0 \cdot 10^{-3}$ & x1.14 \\
 & $2^{16}$ & $5.7 \cdot 10^{-3}$ & x1.59 \\
 & $2^{18}$ & $20.2 \cdot 10^{-3}$ & x5.68 \\
 & $2^{18.5}$ & $76.0 \cdot 10^{-3}$ & x21.38 \\
\midrule[1pt]
\multirow{5}{*}{\parbox{2cm}{\centering VGG-16}}
 & $1$ (SLC) & $9.4 \cdot 10^{-3}$ & --- \\
 & $2^{10}$ & $9.6 \cdot 10^{-3}$ & x1.02 \\
 & $2^{14}$ & $9.9 \cdot 10^{-3}$ & x1.05 \\
 & $2^{16}$ & $11.4 \cdot 10^{-3}$ & x1.20 \\
 & $2^{18}$ & $26.4 \cdot 10^{-3}$ & x2.79 \\
 & $2^{18.5}$ & $80.5 \cdot 10^{-3}$ & x8.52 \\
\midrule[1pt]
\multirow{5}{*}{\parbox{2cm}{\centering Inception-v3}}
 & $1$ (SLC) & $6.0 \cdot 10^{-3}$ & --- \\
 & $2^{10}$ & $6.2 \cdot 10^{-3}$ & x1.03 \\
 & $2^{14}$ & $6.5 \cdot 10^{-3}$ & x1.09 \\
 & $2^{16}$ & $7.3 \cdot 10^{-3}$ & x1.22 \\
 & $2^{18}$ & $18.7 \cdot 10^{-3}$ & x3.11 \\
 & $2^{18.5}$ & $76.6 \cdot 10^{-3}$ & x12.75 \\
\midrule[1pt]
\multirow{5}{*}{\parbox{2cm}{\centering Resnet-50}}
 & $1$ (SLC) & $6.1 \cdot 10^{-3}$ & --- \\
 & $2^{10}$ & $6.4 \cdot 10^{-3}$ & x1.04 \\
 & $2^{14}$ & $6.6 \cdot 10^{-3}$ & x1.08 \\
 & $2^{16}$ & $7.4 \cdot 10^{-3}$ & x1.20 \\
 & $2^{18}$ & $19.1 \cdot 10^{-3}$ & x3.11 \\
 & $2^{18.5}$ & $78.0 \cdot 10^{-3}$ & x12.69 \\
\midrule[1pt]
\multirow{5}{*}{\parbox{2cm}{\centering Resnet-101}}
 & $1$ (SLC) & $8.0 \cdot 10^{-3}$ & --- \\
 & $2^{10}$ & $8.2 \cdot 10^{-3}$ & x1.03 \\
 & $2^{14}$ & $8.3 \cdot 10^{-3}$ & x1.04 \\
 & $2^{16}$ & $9.4 \cdot 10^{-3}$ & x1.19 \\
 & $2^{18}$ & $23.5 \cdot 10^{-3}$ & x2.95 \\
 & $2^{18.5}$ & $80.4 \cdot 10^{-3}$ & x10.10 \\
    \bottomrule\\
  \end{tabular}
\label{tab:resultsthree}
\end{center}
\end{table}
}

\begin{abstract}
We consider neural network training, in applications in which there are many possible classes, but at test-time, the task is a binary classification task of determining whether the given example belongs to a specific class, where the class of interest can be different each time the classifier is applied. For instance, this is the case for real-time image search. We define the \emph{Single Logit Classification} (SLC) task: training the network so that at test-time, it would be possible to accurately identify whether the example belongs to a given class in a computationally efficient manner, based only on the output logit for this class. 
We propose a natural principle, the \emph{\lsplong}, as a guideline for choosing and designing losses suitable for the SLC. 
We show that the cross-entropy loss function is not aligned with the \lsplong. In contrast, there are known loss functions, as well as novel batch loss functions that we propose, which are aligned with this principle. In total, we study seven loss functions. 
Our experiments show that indeed in almost all cases, losses that are aligned with the \lsplong\ obtain at least 20\% relative accuracy improvement in the SLC task compared to losses that are not aligned with it, and sometimes considerably more. Furthermore, we show that fast SLC does not cause any drop in binary classification accuracy, compared to standard classification in which all logits are computed, and yields a speedup which grows with the number of classes. For instance, we demonstrate a 10x speedup when the number of classes is 400,000. 
\texttt{Tensorflow} code for optimizing the new batch losses is publicly
available at \url{https://github.com/cruvadom/Logit_Separation}. 
\end{abstract}


\section{Introduction} \label{sec:intro}
With the advent of Big Data, classifiers can learn fine-grained distinctions, and are used for classification in settings with very large numbers of classes. Datasets with up to hundreds of thousands of classes are already in use in the industry \citep{DBLP:conf/cvpr/DengDSLL009,DBLP:journals/corr/PartalasKBAPGAA15}, and such classification tasks have been studied in several works (e.g., \citet{DBLP:conf/icml/WestonMY13,DBLP:journals/jmlr/GuptaBW14}). 
Classification with a large number of classes appears naturally in vision, in language modeling and in machine translation \citep{DBLP:journals/corr/BahdanauCB14,DBLP:journals/corr/JozefowiczVSSW16,DBLP:conf/cvpr/DeanRSSVY13}.

When using neural network classifiers, one implication of a large number of classes is a high computational burden at test-time. Indeed, in standard neural networks using a softmax layer and the cross-entropy loss, the computation needed for finding the logits of the classes (the pre-normalized outputs of the top network layer) is linear in the number of classes \citep{DBLP:journals/corr/GraveJCGJ16}, and can be prohibitively slow for high-load systems, such as search engines and real-time machine-translation systems.

In many applications, the task at test-time is not full classification of each example into one of the many possible classes. Instead, the task, each time the trained classifier is used, is to identify whether the current example should be classified into one of a small subset of the possible classes, or even a single class. This class can be different every time the classifier is used. 
Consider for example the case of real-time image search \citep{DBLP:conf/iros/MaturanaS15,DBLP:conf/cvpr/RedmonDGF16} from a live feed from multiple cameras. When the user queries for images of object A, the classifier has to process a large number of images, and decide whether each image contains an instance of object A or not. The classifier is then activated for the second time, this time with a query to find images of object B. New images are now processed by the classifier, to determine which ones contain an instance of object B.

The setting described above has various applications in which the object to detect might change every time the model is used. For instance, consider cameras installed on autonomous cars, security cameras for detecting objects of interest, or live face recognition, where a different person is to be identified each time, based on the given query. 

In the setting that we consider, while every use of the classifier at test-time tests for a single class (or a small number of classes), the classifier itself must support queries on \emph{any} of the classes, since it will be used again and again, each time with a different class as a query.
As the number of classes may be large, it is not reasonable to train a separate model for every possible class that might be queried at test time. Instead, our goal is to have a single model which supports \emph{all} possible class-queries. 

For this type of applications, one would ideally like to have a test-time computation that does not depend on the total number of possible classes. A natural approach is to calculate only the logit of the class of interest, and use this value alone to infer whether this is the true class of the example. However, the logit of a single class might only be meaningful in comparison to logits of other classes, in which case, unless the other logits are also computed, it cannot be used to successfully determine whether the test example belongs to the class of interest. We name the goal of inferring class correctness from the logit of that class alone \emph{Single Logit Classification} (SLC). 
Note that SLC is a binary classification task, stressing the fact that only one logit is computed. 
In Figure \ref{fig:time} we demonstrate the speedup yielded by SLC, compared to binary classification in the method which computes all logits and uses them for normalization, as the number of classes increases. For instance, computing only a single logit yields a 10x speedup in evaluation time when there are 400,000 possible classes. The speedup increases with the number of possible classes. See \secref{experiments} for full details on the experimental setting and the resulting speedup.

\begin{figure}[htbp]
\centerline{\includegraphics[width=\columnwidth]{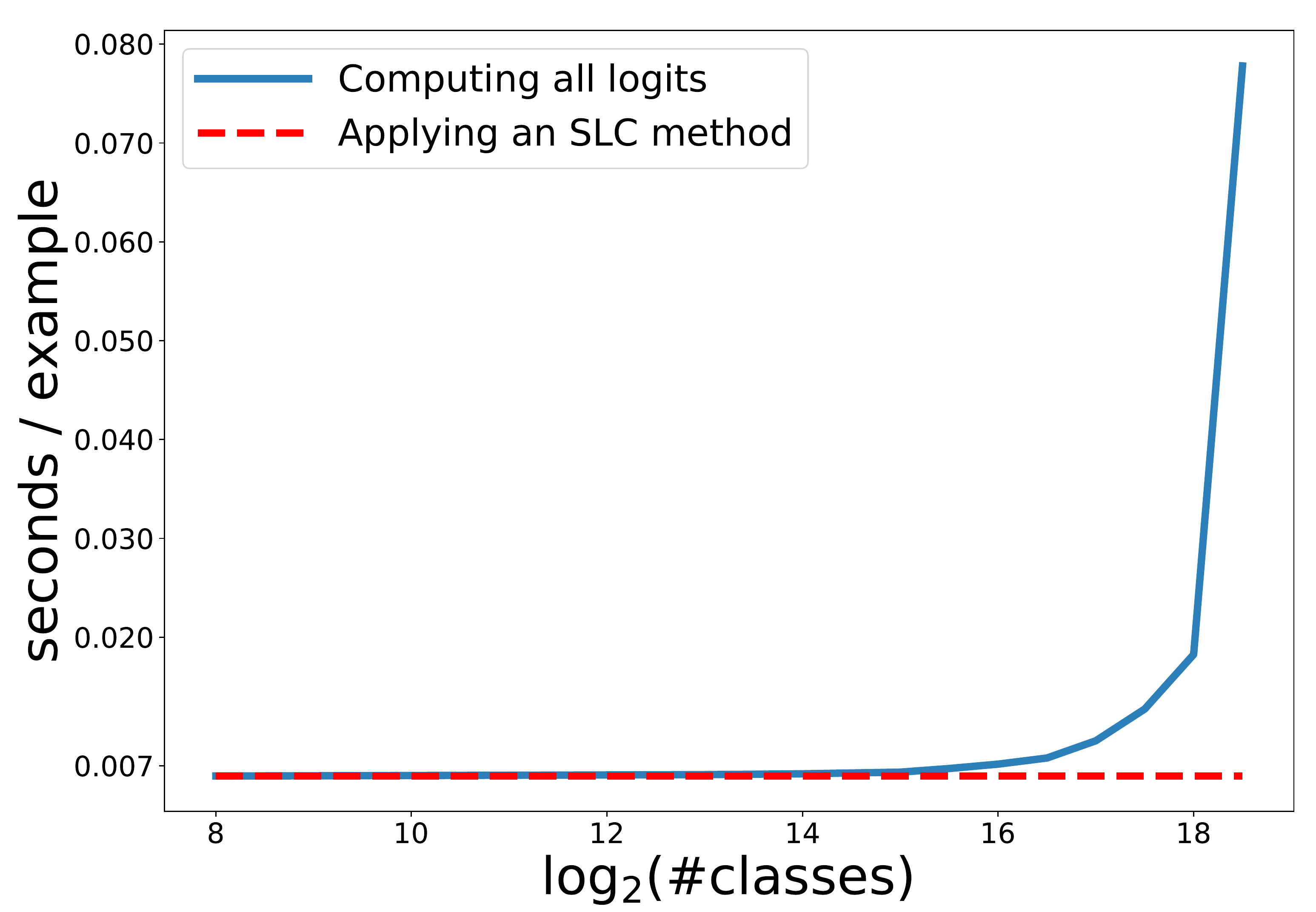}}
\caption{Computation time of the logits from network inputs, using an inception-V3 image classification architecture \citep{DBLP:conf/cvpr/SzegedyVISW16}, where the topmost layer is replaced according to the appropriate number of classes. When applying SLC, computation cost is fixed regardless of the number of classes, which can lead to considerable speedups when the number of classes is large. 
See \secref{experiments} for full details regarding speedups and the experimental setting.}
\label{fig:time}
\end{figure}

In this work, we show that when using the standard cross-entropy loss for training, the value of a single logit is not informative enough for determining whether this is indeed the true class for the example. In other words, the cross-entropy loss yields poor accuracy in the SLC task. Further, we identify a simple principle that we name the \emph{\lsplong}. This principle captures an essential property that a loss function must have in order to yield good accuracy in the SLC task. The principle states that to succeed in the SLC task, the training objective should optimize for the following property: 
\begin{quote}
\emph{The value of any logit that belongs to the correct class of any training example should be larger than the value of any logit that belongs to a wrong class of any (same or other) training example.}
\end{quote}

\begin{figure*}[htbp]
\centerline{\includegraphics[height=1.8in, keepaspectratio]{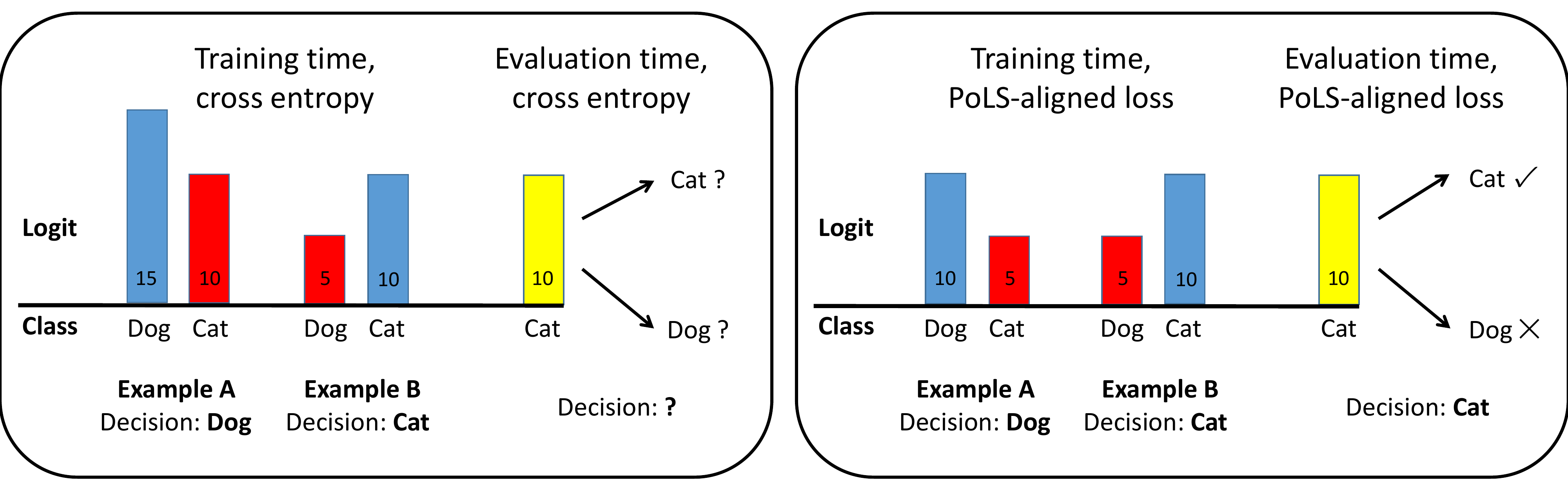}}
\caption{The \lsplong. Left: when training with the cross-entropy loss, the logit values for the class `Cat' can be the same for two examples, one where it is the true class (blue) and one where it is not (red). Therefore, at test-time, a logit with the same value for the class `Cat' does not indicate whether the example belongs to this class. Right: With a loss function that is aligned with the \lsplong, all true logits are greater than all false logits at training time. Hence, at test time, a single logit can indicate the correctness of its respective class.}
\label{fig:pols}
\end{figure*}

We give a formal definition of the \lsplong\ in \secref{pols}. See Figure \ref{fig:pols} for an illustration.
We study previously suggested loss functions and their alignment with the \lsplong. 
We show that the \lsplong\ is satisfied by the self-normalization \citep{DBLP:conf/acl/DevlinZHLSM14} and Noise-Contrastive Estimation \citep{DBLP:conf/icml/MnihT12} training objectives, proposed for calculating posterior distributions in the context of natural language processing, as well as by the binary cross-entropy loss used in multi-label settings \citep{DBLP:conf/cvpr/WangYMHHX16,DBLP:conf/icip/HuangWWT13}. In contrast, the principle is not satisfied by the standard cross-entropy loss and by the max-margin loss. We derive new training objectives for the SLC task based on the \lsplong. These objectives are novel batch versions of the cross-entropy loss and the max-margin loss, and we show that they are aligned with the \lsplong. In total, we study seven different training objectives. \texttt{Tensorflow} code for optimizing the new batch losses is publicly available at \url{https://github.com/cruvadom/Logit_Separation}.

We corroborate in experiments that the \lsplong\ indeed explains the
difference in accuracy of the different loss functions in the SLC task,
concluding that training with a loss function that is aligned with the
\lsplong\ results in logits that are more informative as a
standalone value, and as a result, considerably better SLC accuracy. 
Specifically, objectives that satisfy the \lsplong\ outperform standard objectives such as the cross-entropy loss on the SLC task with at least 20\% relative accuracy improvement in almost all cases, and sometimes considerably more. 
In another set of experiments, we show that when using loss functions that are aligned with the \lsplong, SLC does not cause any decrease in binary classification accuracy for a given class, compared to the case where \emph{all} logits are computed, while keeping the computation cost independent of the total number of classes.
Finally, we perform further experiments to determine the speedup factor one can gain by using SLC instead of computing all logits. We show that when the number of classes is large, considerable speedups are gained. 

We conclude that by designing a training objective according to the \lsplong\ and applying SLC when the number of classes in large, one can gain considerable speedups at test time without any degradation in model accuracy.
As the number of classes in standard datasets has been rapidly growing over the last few years, SLC and the \lsplong\ may play a key role in many applications.


\subsection{Related Work}
We review existing methods that are relevant for faster test-time classification.
The hierarchical softmax layer
\citep{DBLP:conf/aistats/MorinB05}
replaces the flat
softmax layer with a binary tree with classes as leaves, making the
computational complexity of calculating the posterior probability of each
class logarithmic in the number of classes. A drawback of this method is the
additional construction of the binary tree of classes, which requires expert
knowledge or data-driven methods. Inspired by the hierarchical softmax
approach, \citet{DBLP:journals/corr/GraveJCGJ16} exploit unbalanced word
distributions to form clusters that explicitly minimize the average time for
computing the posterior probabilities over the classes. The authors report an
impressive speed-up factor of between $2$ and $10$ for
posterior probability computation, but their computation time still depends on the total number of classes.
Differentiated softmax was introduced in \citet{DBLP:conf/acl/ChenGA16} as a less computationally expensive alternative to the standard softmax mechanism, in the context of neural language models. With differentiated softmax, each class (word) is represented in the last hidden layer using a different dimensionality, with higher dimensions for more frequent classes. This allows a faster computation for less frequent classes. 
However, this method is applicable only for highly unbalanced class distributions.
Several sampling-based approaches were developed in the context of language
modeling, with the goal of approximating the softmax function at
training-time. Notable examples are importance sampling
\citep{DBLP:journals/tnn/BengioS08}, 
negative sampling \citep{DBLP:conf/nips/MikolovSCCD13}, and Noise Contrastive Estimation (NCE) \citep{DBLP:journals/jmlr/GutmannH10,DBLP:conf/icml/MnihT12}. These methods do not necessarily improve the test-time computational burden, however we show below that the NCE loss can be used for the SLC task.


\section{The \lsplong} \label{sec:pols}
In the SLC task, the only information about an example is the output logit of the model for the single class of interest. Therefore, a natural approach to classifying whether the class matches the example is to set a threshold: if the logit is above the threshold, classify the example as belonging to this class, otherwise, classify it as not belonging to the class. We refer to logits that belong to the true classes of their respective training examples as \emph{true logits} and to other logits as \emph{false logits}. For the threshold approach to work well, the values of all true logits should be larger than the value of all false logits across the training sample (in fact, it is enough to separate true and false logits on a class level, but we stick to the stronger assumption in this work). This is illustrated in Figure \ref{fig:pols}. The \lsplong\ (PoLS), which was stated in words in \secref{intro}, captures this requirement. We formalize this principle below.

Let $[k] := \{1,\ldots,k\}$ be the possible class labels. Assume that the training sample is $S  = ((x_1,y_1),\ldots,(x_n,y_n))$, where $x_i \in \reals^d$ are the training examples, and $y_i \in [k]$ are the labels of these examples.
For a neural network model parametrized by $\theta$, we denote by $z^\theta_y(x)$ the value of the logit assigned by the model to example $x$ for class $y$. The \lsplong\ (PoLS) can be formally stated as follows:
\begin{definition}[The \lsplong]\label{def:pols} 
The \emph{\lsplong} holds for a labeled set $S$ and a model $\theta$, if for any $(x,y),(x',y') \in S$ (including the case $x=x',y=y'$) and any $y'' \neq y'$, we have $z^\theta_y(x) > z^\theta_{y''}(x')$.
\end{definition}
The definition assures that every true logit $z^\theta_y(x)$ is larger than every false logit $z^\theta_{y''}(x')$. If this simple principle holds for all train and test examples, it guarantees
perfect accuracy in the SLC task, since all true logits are larger than all false logits. Thus, a good approach for a training
objective for SLC is to attempt to optimize for this principle on the training
set.  For a loss $\ell$, $\ell(S,\theta)$ is the value of the loss on the
training sample using model $\theta$. A loss $\ell$ is aligned with the
\lsplong\ if for any training sample $S$, a small enough value of
$\ell(S, \theta)$ ensures that the requirement in Definition~\ref{def:pols} is
satisfied for the model $\theta$.
 In the following sections we study the alignment with the PoLS of known losses and new losses.

\section{Standard objectives in view of the PoLS} \label{sec:nopols} In this section we show that the cross-entropy loss \citep{hinton1989connectionist}, which is the standard loss function for neural network classifiers (e.g., \citet{DBLP:conf/nips/KrizhevskySH12}) and the multiclass max-margin loss \citep{DBLP:journals/jmlr/CrammerS01}, do not satisfy the PoLS.

\subsection{The Cross-Entropy Loss}
The cross-entropy loss on a single example is defined as
\begin{equation}\label{eq:ce}
\ell(z,y) = -\log(p_y),
\end{equation}
where
\[p_y := \frac{e^{z_y}}{\sum_{j=1}^k e^{z_j}} = \big(\sum _{j=1}^k e^{z_j-z_y}\big)^{-1}.\]
Note that $p_y$ is the probability assigned by the softmax layer.
It is easy to see that the cross-entropy loss does not satisfy the PoLS. Indeed, as the loss depends only on the difference between logits for every example separately, minimizing it guarantees a certain difference between the true and false logits for every example separately, but does not guarantee that all true logits are larger than all false logits in the training set. 
Formally, the following counter-example shows that this loss is not aligned with the PoLS.
Let $S = ((x_1,1),(x_2,2))$ be the training sample, and let $\theta_\alpha$, for $\alpha > 0$, be a model such that $z^{\theta_\alpha}(x_1) = (2\alpha,\alpha)$, and $z^{\theta_\alpha}(x_2) = (-2\alpha,-\alpha)$. 
Then $\ell(S_{\theta_\alpha}) = 2\log(1+e^{-\alpha})$. Therefore for any $\epsilon > 0$, there is some $\alpha > 0$ such that $\ell(S_{\theta_\alpha}) \leq \epsilon$, but $z_2^{\theta_\alpha}(x_1) > z_2^{\theta_\alpha}(x_2)$, contradicting an alignment with PoLS.

\subsection{The Max-Margin Loss}\label{sec:max-margin}
Max-margin training objectives, most widely known for their role in training Support Vector Machines, are used in some cases for training neural networks \citep{DBLP:conf/icml/SocherLNM11,DBLP:journals/corr/JanochaC17}. 
Here we consider the multiclass max-margin loss suggested by \cite{DBLP:journals/jmlr/CrammerS01}, defined as
\begin{equation}\label{eq:maxmargin}
\ell(z,y) = \max (0, \gamma - z_y + \max \limits _{j \neq y} z_j),
\end{equation}
where $\gamma > 0$ is a hyperparameter that controls the separation margin between the true logit and the false logits of the example.
It is easy to see that this loss too does not satisfy the PoLS, since minimizing it again guarantees only a certain difference between the true and false logits for every example separately, and not across the entire training sample. Indeed, consider the same training sample $S$ as defined in the counter-example for the cross-entropy loss above, and the model $\theta_\alpha$ defined there. Setting $\alpha = \gamma$, we have $\ell(S_{\theta_{\gamma}}) = 0$. Thus for any $\epsilon > 0$, $\ell(S_{\theta_{\gamma}}) < \epsilon$, but $z_2^{\theta_\gamma}(x_1) > z_2^{\theta_\gamma}(x_2)$, contradicting an alignment with PoLS.

\section{Objectives that satisfy the PoLS} \label{sec:known}
In this section we consider objectives that have been previously suggested for addressing problems that are somewhat related to the SLC task. We show that these objectives indeed satisfy the PoLS.

\subsection{Self-Normalization} \label{sec:selfnorm} 
Self-normalization \citep{DBLP:conf/acl/DevlinZHLSM14} was introduced in the context of neural
language models, to avoid the costly step of computing the posterior
probability distribution over the entire vocabulary when evaluating the
trained models. The self-normalization loss is a sum of the cross-entropy loss
with an additional term. Let $\alpha > 0$ be a hyperparameter, and $p_y$ as
defined in \eqref{ce}. The self-normalization loss is defined by
\[
\ell(z,y) = -\log (p_y) +\alpha \cdot \log ^2 (\sum \limits _{j=1}^k e^{z_j}).
\]
The motivation for this loss is self-normalization: The second term is minimal
when the softmax normalization term $\sum_{j=1}^k e^{z_j}$ is equal to
$1$. When it is equal to $1$, the exponentiated logit $e^{z_j}$ can be
interpreted as the probability that the true class for the example is
$j$. 
\citet{DBLP:conf/acl/DevlinZHLSM14} report a speed-up by a factor of 15 in
evaluating models trained when using this loss, since the
self-normalization enables computing the posterior probabilities for only a subset of the vocabulary. 

Intuitively, this loss should also be useful for the SLC task: If the softmax normalization term is
always close to $1$, there should be no need to compute it, thus only the logit of the class in question should be required to infer whether this class in the correct one for the example. 
Indeed, we show that the self-normalization loss is aligned with the PoLS. When the first term in the loss is minimized for an example, correct and wrong logits are as different as possible from one another. When the second term is minimized for an example, the sum of exponent logits is equal to one. Therefore, when both terms are minimized for an example, the correct logit converges to zero while wrong logits converge to negative infinity. When this is done for the whole training sample, all correct logits are larger than all wrong logits in the training sample. A formal proof is provided in \appref{selfnorm}.

\subsection{Noise Contrastive Estimation} \label{sec:nce} Noise Contrastive
Estimation (NCE) \citep{DBLP:journals/jmlr/GutmannH10,DBLP:conf/icml/MnihT12}
was considered, like self-normalization, in the context of natural language
learning. This approach was developed to speed up neural-language model
training over large vocabularies. In NCE, the multiclass classification problem is
treated as a set of binary classification problems, one for each class. Each
binary problem classifies, given a context and a word, whether this word is
from the data distribution or from a noise distribution. Using only $t$ words
from the noise distribution (where $t$ is an integer hyperparameter) instead of the entire vocabulary leads to a
significant speedup at training-time. Similarly to the self-normalization
objective, NCE, in the version appearing in \citet{DBLP:conf/icml/MnihT12}, is
known to produce a self-normalized logit vector
\citep{DBLP:conf/naacl/AndreasK15}. This property makes NCE a good candidate
for the SLC task, as single logit values are informative for the class correctness, and not only when compared other logits in the same example.

The loss function used in NCE for a single training example, as given by \cite{DBLP:conf/icml/MnihT12}, is defined based on a distribution over the possible classes, denoted by $q = (q(1),\ldots,q(k))$, where $\sum_{i=1}^k q(i) = 1$. 
The NCE loss, in our notation, is 
\begin{equation}\label{eq:gj} 
\ell(z,y) = -\log g_y -  t\cdot \E_{j \sim q}\left[\log(1-g_j)\right],
\end{equation}
where
\[g_j := (1 + t\cdot q(j)\cdot e^{-z_j})^{-1}\]
During training, the second term in the loss is usually approximated by Monte-Carlo approximation, using $t$ random samples of $j \sim q$, to speed up training time \citep{DBLP:conf/icml/MnihT12}.

We observe that NCE loss is aligned with the PoLS. First, observe that $g_j$ is of a similar form to $\sigma(z_j)$ where $\sigma(z) = (1 + e^{-z})^{-1}$ is the sigmoid function. Therefore, it is easy to see that when the term above is minimized for one example, the value of true logit $z_y$ converges to infinity, and the values of all false logits converge to negative infinity. When the above term is minimized for the entire training set, all true logits are larger than all false logits across the training set. A formal proof is provided in \appref{nce}.

\subsection{Binary Cross-Entropy} \label{sec:binaryce}
The last known loss that we consider is often used in multilabel classification settings. In multilabel settings, each example can belong to several classes, and the goal is to identify the set of classes an example belongs to. A common approach \citep{DBLP:conf/cvpr/WangYMHHX16,DBLP:conf/icip/HuangWWT13} is to try to solve $k$ binary classification problems of the form ``Does $x$ belong to class $j$?'' using a single neural network model, by minimizing the sum of the cross-entropy losses that correspond to these binary problems. 
In this setting, the label of each example is a binary vector $(r_1,\ldots, r_k)$, where $r_j = 1$ if $x$ belongs to class $j$ and 0 otherwise. The loss for a single training example with logits $z$ and label-vector $r$ is
\[
\ell(z,(r_1,\ldots,r_k)) = -\sum \limits _{j=1}^n r_j \log(\sigma(z_j)) + (1-r_j)\log(1-\sigma(z_j)),
\]
where $\sigma(z) = (1 + e^{-z})^{-1}$ is the sigmoid function. This loss can also be used for our setting of multiclass problems, by defining $r_j:=\mathbf{1}_{j = y}$ for an example $(x,y)$. This gives the multiclass loss 
\[
\ell(z,y) = -\log(\sigma(z_y)) + \sum_{j \neq y}\log(1-\sigma(z_j)).
\]

The binary cross-entropy is also aligned with the PoLS. Indeed, similarly to case of the NCE loss, it is easy to see that when the term above is minimized for one example, the value of true logit $z_y$ converges to infinity, and the values of all false logits converge to negative infinity. When the above term is minimized for the entire training set, all true logits are larger than all false logits across the training set. A formal proof is provided in  \appref{binaryce}.

\section{New training objectives for the SLC task}\label{sec:batchlosses}
In this section we propose new training objectives for the SLC task, designed
to satisfy the PoLS. These objectives adapt the training objectives of
cross-entropy and max-margin, studied in \secref{nopols}, that do not satisfy
the PoLS, by generalizing them to optimize over \emph{batches} of training
samples. We show that the revised losses satisfy the PoLS. This approach does
not require any new hyper-parameters, since the batch size is already a
hyperparameter in standard Stochastic Gradient Descent. Further, this allows an easy adaptation of
available neural network implementations to the SLC task. When the
cross-entropy loss or the max-maring loss are minimized, they guarantee a certain difference between the true and the false
logits of each example separately. Our generalization of these losses to
batches of examples enforces an ordering also between true and false logits of
different examples.

\subsection{Batch Cross-Entropy} \label{sec:batchce}
\newcommand{\kl}{\mathrm{KL}} Our first batch loss generalizes the
cross-entropy loss, which was defined in \eqref{ce}. The
cross-entropy loss can be given as the Kullback-Leibler (KL) divergence
between two distributions, as follows. 
The KL divergence between two discrete probability distributions $P$ and $Q$ over
$[k]$ is defined as $\mathrm{KL}(P||Q) := \sum _{i=j}^k P(j) \log (P(j)/Q(j)).$
For an example $(x,y)$, let $P_{(x,y)}$ be the distribution over $[k]$ which deterministically outputs $y$, and let $Q_x$ be the distribution defined by the softmax normalized logits, $Q_x(j) = e^{z_j}/\sum_{i=1}^k e^{z_i}$. Then 
it is easy to see that for $p_y$ as defined in \eqref{ce},
$\mathrm{KL}(P_{(x,y)}||Q_x) = -\log p_y,$
exactly the cross-entropy loss in \eqref{ce}. 

We define a batch version of this loss, using the KL-divergence between distributions over batches.
Recall that the $i$'th example in a batch $B$ is denoted $(x_i,y_i)$. Let $P_{B}$ be the distribution over $[m] \times [k]$ defined by 
\[ 
P_B(i,j) := \begin{cases} 
			\frac{1}{m} & j=y_i,\\
			0 & \text{otherwise}.
			
\end{cases} 
\]
Let $Q_B$ be the distribution defined by the softmax normalized logits over the entire batch $B$. Formally, denote $Z(B) := \sum \limits _{i=1}^m \sum \limits _{j=1}^k e^{z_j(x_i)}$. Then $Q_B(i,j) := e^{z_j(x_i)}/Z(B)$. We then define the batch cross-entropy loss as follows.
\begin{definition}[The batch cross-entropy loss]\label{def:cross}
Let $m > 1$ be an integer, and let $B$ be a uniformly random batch of size $m$ from $S$. The \emph{batch cross-entropy loss} of a training sample $S$ is 
\[
\ell(S) := \E_B[L_c(B)],
\quad\text{ where }\quad
L_c(B) := \mathrm{KL}(P_B || Q_B).
\]
\end{definition}

This batch version of the cross-entropy loss is aligned with the PoLS. Indeed, when this loss is minimized for one training batch, all true logits converge to some positive value (as a normalized exponentiated true logit converges to $1/m$), while all false logits converge to negative infinity (as a normalized exponentiated false logit converges to zero). Therefore, when minimizing this loss across the whole training set, all true logits are larger than all false logits in the training set. A formal proof is provided in  \appref{new}.

\subsection{Batch Max-Margin} 
Our second objective is a batch version of the max-margin loss, which was defined in \eqref{maxmargin}. For a batch $B$, denote the minimal true logit in $B$, and the maximal false logit in $B$, as follows: 
\[ 
z_{+}^B := \min_{(x,y) \in B} z_{y}(x), \quad\text{and}\quad z_{-}^B := \max_{(x,y) \in B, j \neq y} z_{j}(x).
\]
\begin{definition}[The batch max-margin loss]\label{def:margin}
Let $m > 1$ be an integer, and let $B$ be a uniformly random batch of size $m$ from $S$. Let $\ell$ be the single-example max-margin loss defined in \eqref{maxmargin}, let $\gamma > 0$ be the max-margin hyper-parameter. The \emph{batch max-margin} is defined by 
\[\ell(S) := \E_B[L_m(B)],\]
where
\[L_m(B) := \frac{1}{m} \max (0, \gamma - z_{+}^B + z_{-}^B) + \frac{1}{m} \sum \limits _{(x,y) \in B} \ell(z(x),y). \]
\end{definition}
The batch version of the max-margin loss is aligned with the PoLS. Minimizing the first term in the loss makes sure that all true logits in the batch are larger than all false logits in the batch. Therefore, minimizing the loss over the entire training set makes sure that the PoLS holds. A formal proof is provided in \appref{new}. Note that while the seconds term in the loss is not necessary for ensuring alignment with the PoLS, it is necessary for practical reasons, as without it the gradient is propagated through only two logits from the entire minibatch, which leads to harder optimization and poorer generalization.

\section{Experiments}\label{sec:experiments}
We first empirically show that the PoLS plays a dominant role in SLC accuracy, and that all loss functions that are aligned with the PoLS yield considerably better accuracy in SLC compared to loss functions that are not aligned with the PoLS.
We then investigate whether a single logit suffices for obtaining  a good binary classification accuracy, compared to the method of computing all logits and using them for normalization. We find that when using a PoLS-aligned loss function, binary classification accuracy is about the same whether we use a single logit only (SLC) or all logits. Lastly, we evaluate the speedups gained by using SLC for binary classification, instead of computing all logits. We show that the speedup increases with the number of classes. For instance, we demonstrate a 10x speedup when the number of classes is approximately 400,000.

\tabresone 

\subsection{PoLS and SLC Accuracy} \label{sec:exp1}
We tested the SLC tasks on neural networks trained with each of the objectives. 
To evaluate the success of a learned model in the SLC task we measured, for each class $j$ and each threshold $T$, the precision and recall in identifying examples from class $j$ using the test $z_j > T$, and calculated the Area Under the Precision-Recall curve (AUPRC) defined by the entire range of possible thresholds. We also measured the precision at fixed recall values (with dictate the threshold $T$ to use) $0.9$ (Precision@0.9) and $0.99$ (Precision@0.99). We report the averages of these values over all the classes in the dataset. 

We tested five computer-vision classification benchmark datasets (using their built-in train/test splits): MNIST \citep{lecun1998gradient}, SVHN \citep{netzer2011reading} CIFAR-10 and CIFAR-100 \citep{krizhevsky2009learning}. The last dataset is Imagenet \citep{DBLP:journals/ijcv/RussakovskyDSKS15}, which has 1000 classes, demonstrating the scalability of the PoLS approach to many classes. Due to its size, training on Imagenet is highly computationally intensive, therefore we tested only two representative methods for this dataset, which do not require tuning additional hyperparameters. For every dataset, a single network architecture was used for all training objectives. 

The network architectures we used are standard, and were fixed before running the experiments. For the MNIST dataset, we used an MLP comprised of two fully-connected layers with 500 units each, and an output layer, whose values are the logits, with 10 units. For the SVHN, CIFAR-10 and CIFAR-100 datasets, we used a convolutional neural network \citep{lecun1989backpropagation} with six convolutional layers and one dense layer with 1024 units. The first, third and fifth convolutional layers used a $5 \times 5$ kernel, where other convolutional layers used a $1 \times 1$ kernel. The first two convolutional layers were comprised of 128 feature maps, where convolutional layers three and four had 256 feature maps, and convolutional layers five and six had 512 feature maps. Max-pooling layers with $3 \times 3$ kernel size and a $2 \times 2$ stride were applied after the second, fourth and sixth convolutional layers. In all networks, batch normalization \citep{DBLP:conf/icml/IoffeS15} was applied to the output of every fully-connected or convolutional layer, followed by a rectified-linear non-linearity. For every combination of a training objective and a dataset (with its fixed network architecture), we optimized for the best learning rate among  $1, 0.1, 0.01, 0.001$ using the classification accuracy on a validation set. Except for Imagenet, each model was trained for $10^5$ steps, which always sufficed for convergence. For the Imagenet experiments, we used an inception-v3 architecture \citep{DBLP:conf/cvpr/SzegedyVISW16} as appears in the \texttt{tensorflow} 
repository. 
We used all the default hyperparameters from this implementation, changing only the loss function used. For every tested loss function, we trained the inception-v3 model for $6 \cdot 10^6$ iterations.

Experiment results are reported in Table \ref{tab:resultsone}. Since many of the measures in our experiments are close to their maximal value of $1$, we report the value of \emph{one minus} each measure, so that a smaller number indicates a better accuracy. 
For each dataset, the losses above the dashed line do not satisfy the PoLS (\secref{nopols}), while the losses below the line do (Sections \ref{sec:known} and \ref{sec:batchlosses}). 
In the table, the best result for each dataset and measure is indicated in boldface.
Finally, the bottom row in each dataset stands for the mean relative improvement between PoLS-aligned losses and other losses, i.e., the relative improvement of the mean of PoLS-aligned losses for a given measure, compared to the mean of losses that are not aligned with the PoLS.

From the results in Table \ref{tab:resultsone}, it can be seen that the mean relative improvement of training objectives that are aligned with the PoLS compared to non-aligned objectives is usually at least 20\%, and in many cases considerably more. We conclude from these experiments that indeed, alignment with the PoLS is a crucial ingredient for success in the SLC task.


\subsection{SLC vs Computing All Logits} \label{sec:exp2}
To investigate whether using a single logit (SLC) degrades binary classification accuracy, we compared the binary classification accuracy obtained by SLC with the one obtained by using all logits. In the latter case, we used \emph{all} output logits of the cross-entropy loss after softmax normalization, which is computationally expensive compared to SLC. The experiment setting and binary classification accuracy measures are identical to the one used in Section \ref{sec:exp1}. 

Results are presented in Table \ref{tab:resultstwo}. The first two rows present results for binary classification using a single logit (SLC): The first row reports the mean result for loss functions that are aligned with the PoLS reported in Table \ref{tab:resultsone}. The second row reports the results for the batch cross-entropy, which is the PoLS-aligned loss that obtained the best accuracy in Table \ref{tab:resultsone}. The third row reports results for the cross-entropy method, in which all logits are computed and used for normalization.

\tabrestwo

The results in Table \ref{tab:resultstwo} show that cross-entropy with all logits yields results that are comparable to the results obtained with a single logit, and specifically, with the batch cross-entropy: none of the approaches is consistently more successful in classification than the other. Since calculating cross-entropy with all logits is more computationally demanding, we conclude that in our setting of binary classification with multiple classes at test time, SLC, and specifically batch cross-entropy, is an attractive alternative to standard cross-entropy with all logits. 

\subsection{SLC Speedups}
We estimated the speedups gained by performing SLC, compared to methods in which all logits are computed. 
We used five prominent image classification architectures (Alexnet \citep{DBLP:conf/nips/KrizhevskySH12}, VGG-16 \cite{DBLP:journals/corr/SimonyanZ14a}, Inception-v3 \citep{DBLP:conf/cvpr/SzegedyVISW16}, Resnet-50 and Resnet-101 \citep{DBLP:conf/cvpr/HeZRS16}). As we are interested in test-time performance, we only measure the time required for computing the forward-pass of a given network. To measure SLC computation time, we replace the top layer by a layer with single unit and measure the time to compute the single logit given an input to the network. To measure the computation time when computing the logits of $k$ classes, we replace the top layer with a layer containing $k$ units, and again measure the time it takes to compute all logits, given an input to the network.

The computation time of a model generally does not depend on its input data or on its accuracy. The time to compute a forward pass of a given model is practically identical whether the input data is random noise or originates from a real dataset of the same dimensions and data range. Therefore, in these experiments we use random noise as input to the networks, and the networks themselves are randomly initialized and not trained. Computation is done using \texttt{Tensorflow} and a single NVIDIA Maxwell Titan-X GPU, and forward-pass computation time per example is averaged across 100 minibatches of 32 examples. We use the public implementation of all architectures, as appears in the \texttt{tensorflow} repository.

The timing results are given in Table
\ref{tab:resultsthree}. For each network architecture, the first row reports
the forward-pass computation time with a single logit. The following rows
correspond to different numbers of classes. We report the
forward-pass computation time, as well as the speedup obtained by
using SLC for this number of classes. This speedup is calculated as the ratio
between the computation time for this number of classes and the computation
time for SLC (the first row).  As expected, the results show a speedup for all
architectures, with larger speedups when there are more classes. 

For networks with up to $2^{14}=16384$ classes, the speedup is relatively
small, since computation of the network layers other than the logit layer dominates the forward-pass computation time. In contrast, when there are many classes, the computation of logits dominates the forward-pass
computation time. Hence, SLC obtains a x2.8-x.5.7 speedup for
$2^{18} = 262144$ classes, and x8.5-x21.3 speedup for $2^{18.5} = 370727$
classes. 

In our experiments in Sections \ref{sec:exp1} and \ref{sec:exp2}, we showed that our findings scale well from 10 to 100 and 1000 classes, and we expect these results and findings to scale further to models with a larger number of classes. 
Ideally, we would have directly tested datasets with hundreds of thousands of classes, to show that the results from Sections \ref{sec:exp1} and \ref{sec:exp2} scale to datasets with this many classes. However, since such datasets are very large (for instance, the Imagenet-21K dataset has $21,\!000$ classes and more than 14 million examples), these experiments were infeasible with our computational resources. In comparison, a single Inception-V3 model for the significantly smaller Imagenet dataset ($\sim$ 1 million examples), took approximately three weeks to train on a single GPU. 


We conclude from this set of experiments that using SLC instead of computing all logits can result in considerable speedups, that grow with the number of classes. 

\tabresthree


\section{Conclusion} 
We considered the Single Logit Classification (SLC) task, which is important in various applications. 
We formulated the \lsplong\ (PoLS), and studied its alignment with seven loss functions, 
including the standard cross-entropy loss and two novel loss functions. 
We established, and corroborated in experiments, that PoLS-aligned loss functions yield more class logits that are more useful for binary classification. 
We further demonstrated 
that training with a PoLS-aligned loss function and applying SLC leads to considerable speedups when there are many classes, with no degradation in accuracy. 
Recent years have seen a constant increase in the number of classes in datasets from various domains, thus we expect SLC and the PoLS to play a key role in applications.

Future work plans include extending the scope the \lsplong\ by applying it to other training mechanisms that do not involve loss functions \cite{keren2017tunable}.

\section*{Acknowledgment}
This work has been supported by the European Community’s Seventh Framework Programme through the ERC Starting Grant No. 338164 (iHEARu). Sivan Sabato was supported in part by the Israel Science Foundation (grant No. 555/15).

\appendix

\section{Ommitted proofs: alignment with the \lsplong}
All the considered losses are a function of the output logits and the example labels. For a network model $\theta$, denote the vector of logits it assigns to example $x$ by $z^\theta(x) = (z^\theta_1(x),\ldots,z^\theta_k(x))$. When $\theta$ and $x$ are clear from context, we write $z_j$ instead of $z_j^\theta(x)$. Denote the logit output of the sample by $S_\theta  = ((z^\theta(x_1),y_1),\ldots,(z^\theta(x_n),y_n))$. A loss function  
$\ell:\cup_{n=1}^\infty (\reals^k \times [k])^n  \rightarrow \reals_+$ assigns a loss to a training sample based on the output logits of the model and on the labels of the training examples. The goal of training is to find a model $\theta$ which minimizes $\ell(S_\theta) \equiv \ell(S,\theta)$.
In almost all the losses we study,  the loss on the training sample is the sum over all examples of a loss defined on a single example: $\ell(S_\theta) \equiv \sum_{i=1}^n \ell(z^\theta(x_i),y_i)$, thus we only define $\ell(z,y)$. We explicitly define $\ell(S_\theta)$ below only when this is not the case.

\subsection{Self-normalization} \label{ap:selfnorm}
We prove that the self-normalization loss satisfies the PoLS: Let a training sample $S$ and a neural network model $\theta$, and consider an example $(x,y) \in S$. We consider the two terms of the loss in order. First, consider $-\log(p_y)$. From the definition of $p_y$ (Eq. \ref{eq:ce}) we have that 
\[
- \log(p_y) = \log(\sum _{j=1}^k e^{z_j-z_y}) = \log(1+ \sum _{j\neq y} e^{z_j-z_y}).
\]
Set $\epsilon_0 := \log(1+e^{-2})$. Then, if $-\log (p_y) < \epsilon_0$, we have 
$\sum _{j\neq y} e^{z_j-z_y} \leq e^{-2}$, which implies that (a) $\forall j \neq y, z_j \leq z_y - 2$ and (b) $e^{z_y} \geq \sum _{j=1}^k e^{z_j}/(1+e^{-2}) \geq \frac12 \sum _{j=1}^ke^{z_j}.$
Second, consider the second term. There is an $\epsilon_1 > 0$ such that if $\log ^2 (\sum_{j=1}^k e^{z_j}) < \epsilon _1$ then (c) $2e^{-1} < \sum  _{j=1}^k e^{z_j} < e$,
which implies $e^{z_y} < e$ and hence (d) $z_y < 1$. 
Now, let $\theta$ such that $\ell(S_\theta) \leq \epsilon := \min(\epsilon_0, \epsilon_1)$. Then $\forall (x,y) \in S$, $\ell(z^\theta(x),y) \leq \epsilon$. From (b) and (c), $e^{-1} < \frac{1}{2} \sum _{j=1}^k e^{z_j} < e^{z^y}$, hence $z_y > -1$. Combining with (d), we get $-1 < z_y < 1$. Combined with (a), we get that for $j \neq y$, $z_j < -1$. To summarize, $\forall (x,y),(x',y') \in S$ and $\forall y'' \neq y'$, we have that $z_y^\theta(x) > -1 > z_{y''}^\theta(x')$, implying PoLS-alignment.

\subsection{Noise-Contrastive Estimation} \label{ap:nce}
Recall the definition of the NCE loss from \eqref{gj}:
\[\ell(z,y) = -\log g_y -  t\cdot \E_{j \sim q}\left[\log(1-g_j)\right]
\]
 where $g_j := (1 + t\cdot q(j)\cdot e^{-z_j})^{-1}.$
We prove that the NCE loss satisfies the PoLS: $g_j$ is monotonic increasing in $z_j$. Hence, if the loss is small, $g_y$ is large and $g_j$ for $j \neq y$, is small. 
Formally, fix $t$, and let a training
sample $S$. There is an $\epsilon_0 > 0$ such that if $-\log g_j \leq \epsilon_0$, then $z_j > 0$. 
Also, there is an $\epsilon_1 > 0$ (which depends on $q$) such that if $- \E_{j \sim q}\left[\log(1-g_j)\right]\leq \epsilon_1$ then $\forall j \neq y$, $\log(1-g_j)$ must be small enough so that $z_j < 0$. 
Now, consider $\theta$ such that $\ell(S_\theta) \leq \epsilon := \min(\epsilon_0, \epsilon_1)$. Then for every $(x,y) \in S$, $\ell(z^\theta(x),y) \leq \epsilon$. This implies that for every $(x,y),(x',y') \in S$ and $y'' \neq y'$, we have that $z^\theta_y(x) > 0 > z^\theta_{y''}(x')$, thus this loss is aligned with the PoLS.

\subsection{Binary cross-entropy} \label{ap:binaryce}
This loss is similar in form to the NCE loss: for $g_j$ as in \eqref{gj}, $g_j = \sigma(z_j - \ln(t \cdot q(j)))$. Since $\sigma(z_j)$ is monotonic, the proof method for NCE carries over and thus the binary cross-entropy loss satisfies the PoLS as well. 

\subsection{Batch Losses}  \label{ap:new}
Recall that the batch losses are defined as $\ell(S_\theta) := \E_B[L(B_\theta)]$, where $B_\theta$ is a random batch out of $S_\theta$ and $L$ is $L_c$ for the cross entropy (Definition~\ref{def:cross}), and $L_m$ is the max-margin loss (Definition~\ref{def:margin}).
If true logits are greater than false logits in every batch separately when using, then the PoLS is satisfied on the whole sample, since every pair of examples appears together in some batch. The following lemma formalizes this: 
\begin{lemma}\label{lem:batch}
If $L$ is aligned with the PoLS, and $\ell$ is defined by $\ell(S_\theta) := \E_B[L(B_\theta)]$, then $\ell$ is also aligned with the PoLS. 
\end{lemma}
\begin{proof}
Assume a training sample $S$ and a neural network model $\theta$. Since $L$ is aligned with the PoLS, there is some $\epsilon' > 0$ such if $L(B_\theta) < \epsilon'$, then for each $(x,y),(x',y') \in B$ and $y'' \neq y'$ we have that $z_y^\theta(x) > z_{y''}^\theta(x')$. Let $\epsilon = \epsilon'/\binom{n}{m}$, and assume $\ell(S_\theta) < \epsilon$. Since there are $\binom{n}{m}$ batches of size $m$ in $S$, this implies that for every batch $B$ of size $m$, $L(B_\theta) \leq \epsilon'$. 
For any $(x,y),(x',y') \in S$, there is a batch $B$ that includes both examples. Thus, for $y'' \neq y'$, $z_y^\theta(x) > z_{y''}^\theta(x')$. Since this holds for any two examples in $S$, $\ell$ is also PoLS-aligned.
\end{proof}


\subsubsection{Batch cross-entropy}
To show that the batch cross-entropy satisfies the PoLS, we show that $L_c$ does, which by \lemref{batch} implies this for $\ell$. 
By the continuity of $\kl$, and since for discrete distributions,
$\kl(P||Q) =0 \iff P \equiv Q$, there is an $\epsilon > 0$ such that if $L(B_\theta) \equiv \mathrm{KL}(P_B || Q^\theta_B)] < \epsilon$, then for all $i,j$,
$|P_B(i,j) - Q^\theta_B(i,j)| \leq \frac{1}{2m}$.
Therefore, for each example $(x,y) \in B$,
\[ 
\frac{e^{z_y^\theta(x)}}{Z(B)} > \frac{1}{2m}, \qquad\text{ and }\qquad \forall j \neq y, \quad \frac{e^{z_j^\theta(x)}}{Z(B)} < \frac{1}{2m}.
\]
It follows that for any two examples $(x,y),(x',y') \in B$, if $y \neq y'$, then $z_y^\theta(x) > \frac{1}{2m} > z_{y'}^\theta(x')$. Therefore $L$ satisfies the PoLS, which completes the proof.

\subsubsection{Batch max-margin}
To show that the batch max-margin loss satisfies the PoLS, we show this for $L_m$ and invoke \lemref{batch}.
Set $\epsilon = \gamma/m$. If $L(B_\theta) < \epsilon$, then $\gamma -z_{+}^B + z_{-}^B < \gamma$, implying $z_{+}^B > z_{-}^B$. Hence, any $(x,y),(x',y') \in B$ such that $y \neq y'$ satisfy $z^\theta_y(x) \geq z_{+}^B > z_{-}^B \geq z^\theta_y(x')$. Thus $L$ is aligned with the PoLS, implying the same for $\ell$.


\bibliographystyle{IEEEtran}
\bibliography{pols}

\end{document}